\documentclass{article}
\usepackage[margin = 1in]{geometry}
\usepackage[numbers, compress, sort]{natbib}
\usepackage[utf8]{inputenc} 
\usepackage{lmodern}
\usepackage[T1]{fontenc}    
\usepackage{hyperref}       
\usepackage{url}            
\usepackage{booktabs}       
\usepackage{amsfonts}       
\usepackage{nicefrac}       
\usepackage{microtype}      
\usepackage{amsthm, amsmath, amssymb}
\usepackage{thm-restate}
\usepackage[linesnumbered, ruled]{algorithm2e}
\usepackage{algpseudocode}
\usepackage{bbm}
\newcommand{\E}{\mathbb{E}}
\newcommand{\Prob}{\mathbb{P}}
\newcommand{\hist}{\mathcal{H}_{k-1}}
\newcommand{\regret}{{\rm Regret}}
\newcommand{\Ac}{\mathcal{A}}
\newcommand{\Sc}{\mathcal{S}}
\newcommand{\Mc}{\mathcal{M}}

\newcommand{\Rc}{\mathcal{R}}
\newcommand{\Hc}{\mathcal{H}}

\DeclareMathOperator*{\argmin}{argmin}

\def\ind{\mathbbm{1}}

\newtheorem{thm}{Theorem}
\newtheorem{lemma}[thm]{Lemma}

\newtheorem{remark}{Remark}

\title{Worst-Case Regret Bounds for Exploration via Randomized Value Functions}

\author{%
  Daniel Russo\\
  Columbia University\\
  \texttt{djr2174@gsb.columbia.edu} \\
}

\begin{document}

\maketitle

\begin{abstract}
	This paper studies a recent proposal to use randomized value functions to drive exploration in reinforcement learning. These  randomized value functions are generated by injecting random noise into the training data, making the approach compatible with many popular methods for estimating parameterized value functions. By providing a worst-case regret bound for tabular finite-horizon Markov decision processes, we show that planning with respect to these randomized value functions can induce provably efficient exploration. 
\end{abstract}

\section{Introduction}
Exploration is one of the central challenges in reinforcement learning (RL). A large theoretical literature treats exploration in simple finite state and action MDPs, showing that it is possible to efficiently learn a near optimal policy through interaction alone \cite{kearns2002near, brafman2002r, kakade2003sample,strehl2006pac, strehl2009reinforcement, jaksch2010near, dann2015sample, azar2017minimax, dann2017unifying, jin2018q}.
 Overwhelmingly, this literature focuses on optimistic algorithms, with most algorithms explicitly maintaining uncertainty sets that are likely to contain the true MDP. 

It has been difficult to adapt these exploration algorithms to the  more complex problems investigated in the applied RL literature. 
Most applied papers seem to generate exploration through $\epsilon$--greedy or Boltzmann exploration. Those simple methods are compatible with practical value function learning algorithms, which use parametric approximations to value functions to generalize across high dimensional state spaces. Unfortunately, such exploration algorithms can fail catastrophically in simple finite state MDPs \cite[See e.g.][]{osband2017deep}. This paper is inspired by the search for principled exploration algorithms that both (1) are compatible with practical function learning algorithms and (2) provide robust performance, at least when specialized to simple benchmarks like tabular MDPs. 

Our focus will be on methods that generate exploration by planning with respect to randomized value function estimates. This idea was first proposed in a conference paper by \cite{osband2016generalization} and is investigated more thoroughly in the journal paper \cite{osband2017deep}. It is inspired by work on posterior sampling for reinforcement learning (a.k.a Thompson sampling) \cite{strens2000bayesian, osband2013more}, which could be interpreted as sampling a value function from a posterior distribution and following the optimal policy under that value function for some extended period of time before resampling. A number of papers have subsequently investigated approaches that generate  randomized value functions in complex reinforcement learning problems \cite{osband2016deep, azizzadenesheli2018efficient, fortunato2018noisy, burda2019exploration, osband2018randomized, touati2018randomized, tziortziotis2019randomised}. Our theory will focus on a specific approach of \cite{osband2016generalization, osband2017deep}, dubbed \emph{randomized least squares value iteration} (RLSVI), as specialized to tabular MDPs. The name is a play on the classic least-squares policy iteration algorithm (LSPI) of \cite{lagoudakis2003least}. RLSVI generates a randomized value function (essentially) by judiciously injecting Gaussian noise into the training data and then applying applying LSPI to this noisy dataset. One could naturally apply the same template while using other value learning algorithms in place of LSPI. 

This is a strikingly simple algorithm, but providing rigorous theoretical guarantees has proved challenging. One challenge is that, despite the appealing conceptual connections, there are significant subtleties to any precise link between RLSVI and posterior sampling. The issue is that posterior sampling based approaches are derived from a true Bayesian perspective in which one maintains beliefs over the underlying MDP. The approaches of \cite{osband2017deep, touati2018randomized, osband2018randomized, azizzadenesheli2018efficient, fortunato2018noisy,  tziortziotis2019randomised,burda2019exploration} model only the value function, so Bayes rule is not even well defined.\footnote{The precise issue is that, even given a prior over value functions, there is no likelihood function. Given and MDP, there is a well specified likelihood of transitioning from state $s$ to another $s'$, but a value function does not specify a probabilistic data-generating model. } The work of \cite{osband2016generalization, osband2017deep} uses stochastic dominance arguments to relate the value function sampling distribution 
of RLSVI to a correct posterior in a Bayesian model where the true MDP is randomly drawn. This gives substantial insight, but the resulting analysis is not entirely satisfying as a robustness guarantee. It bounds regret on average over MDPs with transitions kernels drawn from a particular Dirichilet prior, but one may worry that hard reinforcement learning instances are extremely unlikely under this particular prior.  

This paper develops a very different proof strategy and provides a worst-case regret bound for RLSVI applied to tabular finite-horizon MDPs. The crucial proof steps are to show that each randomized value function sampled by RLSVI has a significant probability of being optimistic (see Lemma \ref{lem: optimism}) and then to show that from this property one can reduce regret analysis to concentration arguments pioneered by \cite{jaksch2010near} (see Lemmas \ref{lem: reduction to prediction}, \ref{lem: final regret decomposition}). This approach is inspired by frequentist analysis of Thompson sampling for linear bandits \cite{agrawal2013thompson} and especially the lucid description of \cite{abeille2017linear}. However, applying these ideas in reinforcement learning appears to require novel analysis. The only prior extension of these proof techniques to tabular reinforcement learning was carried out by \cite{agrawal2017optimistic}. Reflecting the difficulty of such analyses, that paper does not provide regret bounds for a pure Thompson sampling algorithm; instead their algorithm samples many times from the posterior to form an optimistic model, as in the BOSS algorithm \cite{asmuth2009bayesian}.
Also, unfortunately there is a significant error that paper's analysis and the correction has not yet been posted online, making a careful comparison difficult at this time. 

 The established regret bounds are not state of the art for tabular finite-horizon MDPs. A final step of the proof applies techniques of \cite{jaksch2010near}, introducing an extra $\sqrt{S}$ in the bounds. I hope some smart reader can improve this by intelligently adapting the techniques of \cite{azar2017minimax, dann2017unifying}. However, the primary goal of the paper is not to give the tightest possible regret bound, but to broaden the set of exploration approaches known to satisfy polynomial worst-case regret bounds. To this author, it is both fascinating and beautiful that carefully adding noise to the training data generates sophisticated exploration and proving this formally is worthwhile.


\section{Problem formulation}

We consider the problem of learning to optimize performance through repeated interactions with an unknown finite horizon MDP $M=(H,\Sc, \Ac, P, \Rc, s_1)$. The agent interacts with the environment across $K$ episodes. Each episode proceeds over $H$ periods, where for period $h\in \{1,\ldots, H\}$ of episode $k$ the agent is in state $s_h^k\in \Sc=\{1,\ldots, S\}$, takes action $a^k_h\in \Ac=\{1,\ldots, A \}$, observes the reward $r^{k}_h\in [0,1]$ and, for $h<H$, also observes next state $s^{k}_{h+1}\in \Sc$. 
Let $\hist=\{(s^{i}_h, a^{i}_h, r^{i}_h) : h=1,\ldots H, i=1,\ldots, k-1\}$ denote the history of interactions prior to episode $k$. The Markov transition kernel $P$ encodes the transition probabilities, with 
\[
P_{h,s^k_h,a^k_h}(s) =\Prob(s^k_{h+1} =s \mid  a^k_{h}, s^{k}_{h}, \ldots, a^k_{1}, s^{k}_{1}, \hist  ).
\]
The reward distribution is encoded in $\Rc$, with 
\[
\Rc_{h,s^k_h,a^k_h}(dr)=\Prob\left(r^k_h=dr \mid  a^k_{h}, s^{k}_{h}, \ldots, a^k_{1}, s^{k}_{1}, \hist\right).
\]
We usually instead refer to expected rewards encoded in a vector $R$ that satisfies $R_{h,s,a}=\E[r^k_{h,s,a} \mid s^k_h=s, a^k_h=a]$. We then refer to an MDP $(H,\Sc, \Ac, P, R, s_1)$, described in terms of its expected rewards rather than its reward distribution, as this is sufficient to determine the expected value accrued by any policy. The variable $s_1$ denotes a deterministic initial state, and we assume $s^{k}_1=s_1$ for every episode $k$. At the expense of complicating some formulas, the entire paper could also be written assuming initial states are drawn from some distribution over $\Sc$,which is more standard in the literature.  

A deterministic Markov policy $\pi=(\pi_1,\ldots, \pi_H )$ is a sequence of functions,  where each $\pi_h: \Sc \to \Ac$ prescribes an action to play in each state. We let $\Pi$ denote the space of all such policies. We use $V^{\pi}_h\in \mathbb{R}^{S}$ to denote the value function associated with policy $\pi$ in the sub-episode consisting of periods $\{h, \ldots, H\}$. To simplify many expressions, we set $V_{H+1}^{\pi} = 0\in \mathbb{R}^S$. Then the value functions for $h\leq H$ are the unique solution to the the Bellman equations 
\[
V^{\pi}_h(s)= R_{h,s,\pi(s)} + \sum_{s \in \Sc}  P_{s,h,\pi(s)}(s') V_{h+1}^{\pi}(s') \qquad s\in \Sc,\, h=1,\ldots, H.
\]
The optimal value function is $V^{*}_{h}(s)= \max_{\pi \in \Pi} V^{\pi}_h(s)$. 

An episodic reinforcement learning algorithm \texttt{Alg} is a possibly randomized procedure that associates each history with a policy to employ throughout the next episode. Formally, a randomized algorithm can depend on random seeds $\{\xi_k\}_{k\in \mathbb{N}}$ drawn independently of the past from some prespecified distribution. Such an episodic reinforcement learning algorithm selects a policy $\pi_{k}=\mathtt{Alg}(\hist, \xi_k)$ to be employed throughout episode $k$. 

The cumulative expected regret incurred by \texttt{Alg} over $K$ episodes of interaction with the MDP $M$ is 
\[
{\rm Regret}(M, K, \mathtt{Alg}) = \E_{\mathtt{Alg}}\left[ \sum_{k=1}^{K} V_{1}^{*}(s^k_1) - V_{1}^{\pi_k}(s^k_1)  \right]  
\]
where the expectation is taken over the random seeds used by a randomized algorithm and the randomness in the observed rewards and state transitions that influence the algorithm's chosen policy. This expression captures the algorithm's cumulative expected shortfall in performance relative to an omniscient benchmark, which knows and always employs the true optimal policy. 

Of course, regret as formulated above depends on the MDP $M$ to which the algorithm is applied. Our goal is not to minimize regret under a particular MDP but to provide a guarantee that holds uniformly across a class of MDPs. This can be expressed more formally by considering a class $\Mc$ containing all MDPs with $S$ states, $A$ actions, $H$ periods, and rewards distributions bounded in $[0, 1]$. Our goal is to bound the worst-case regret
$
\sup_{M\in \Mc} {\rm Regret}(M, K, \mathtt{Alg})
$
incurred by an algorithm throughout $K$ episodes of interaction with an unknown MDP in this class. We aim for a bound on worst-case regret that scales sublinearly in $K$ and has some reasonable polynomial dependence in the size of state space, action space, and horizon. We won't explicitly maximize over $M$ in the analysis. Instead, we fix an arbitrary MDP $M$ and seek to bound regret in a way that does not depend on the particular transition probabilities or reward distributions under $M$. 

It is worth remarking that, as formulated, our algorithm knows $S,A$, and $H$ but does not have knowledge of the number of episodes $K$. Indeed, we study a so-called \emph{anytime} algorithm that has good performance for all sufficiently long sequences of interaction. 

\paragraph{Notation for empirical estimates.}  We define $n_{k}(h,s,a) = \sum_{\ell=1}^{k-1}\ind\{(s^\ell_h,a^\ell_h)=(s,a)\}$ to be the number of times action $a$ has been sampled in state $s$, period $h$. For every tuple $(h,s,a)$ with $n_{k}(h,s,a)>0$, we define the empirical mean reward and empirical transition probabilities up to period $h$ by 
\begin{align}\label{eq: empirical reward}
\hat{R}^{k}_{h,s,a}&=\frac{1}{n_{k}(h,s,a)}\sum_{\ell=1}^{k-1}\ind\{(s^\ell_h,a^\ell_h)=(s,a) \} r^\ell_h \\\label{eq: empirical transitions} \hat{P}^{k}_{h,s,a}(s')&=\frac{1}{n_{k}(h,s,a)}\sum_{\ell=1}^{k-1}\ind\{(s^\ell_h,a^\ell_h, s^\ell_{h+1})=(s,a, s') \} \quad \forall s'\in \Sc.
\end{align}
If $(h,s,a)$ was never sampled before episode $k$, we define $\hat{R}^{k}_{h,s,a}=0$ and ${P}^{k}_{h,s,a}=0\in \mathbb{R}^S$.

\section{Randomized Least Squares Value Iteration}
This section describes an algorithm called Randomized Least Squares Value Iteration (RLSVI).  We describe RLSVI as specialized to a simple tabular problem \emph{in a way that is most convenient for the subsequent theoretical analysis}. A mathematically equivalent definition -- which defines RSLVI as estimating a value function on randomized training data -- extends more gracefully . This interpretation is given at the end of the section and more carefully in \cite{osband2017deep}. 

At the start of episode $k$, the agent has observed a history of interactions $\hist$. Based on this, it is natural to consider an estimated MDP $\hat{M}^{k} = (H, \Sc, \Ac, \hat{P}^k, \hat{R}^k, s_1)$ with empirical estimates of mean rewards and transition probabilities. These are precisely defined in Equation \eqref{eq: empirical transitions} and the surrounding text. We could use backward recursion to solve for the optimal policy and value functions under the empirical MDP, but applying this policy would not generate exploration.

RLSVI builds on this idea, but to induce exploration  it judiciously adds Gaussian noise before solving for an optimal policy.  We can define RLSVI concisely as follows. In episode $k$ it samples a random vector with independent components $w^{k} \in \mathbb{R}^{HSA}$, where $w^{k}(h,s,a) \sim N\left( 0, \sigma_{k}^2(h,s,a)  \right)$. We define $\sigma_{k}(h,s,a)= \sqrt{\frac{\beta_{k}}{n_k(h,s,a)+ 1}}$, where $\beta_k$ is a tuning parameter and the denominator shrinks like the standard deviation of the average of $n_{k}(h,s,a)$ i.i.d samples. Given $w^k$, we construct a randomized perturbation of the empirical MDP $\overline{M}^k = (H, \Sc, \Ac, \hat{P}^k, \hat{R}^k + w^k, s_1)$ by adding the Gaussian noise to estimated rewards. RLSVI solves for the optimal policy $\pi^k$ under this MDP and applies it throughout the episode. This policy is, of course, greedy with respect to the (randomized) value functions under $\overline{M}^k$. The random noise $w^k$ in RLSVI should be large enough to dominate the error introduced by performing a noisy Bellman update using $\hat{P}^k$ and $\hat{R}^k$. We set $\beta_k =\tilde{O}(H^3)$ in the analysis, where functions of $H$ offer a coarse bound on quantities like the variance of an empirically estimated Bellman update. For $\beta=\{\beta_{k}\}_{k\in \mathbb{N}}$, we denote this algorithm by $\mathtt{RLSVI}_{\beta}$.

\paragraph{RLSVI as regression on perturbed data.} To extend beyond simple tabular problems, it is fruitful to view RLSVI--like in Algorithm \ref{alg: RLSVI}--as an algorithm that performs recursive least squares estimation on the state-action value function.  Randomization is injected into these value function estimates by perturbing observed rewards and by regularizing to a randomized prior sample. This prior sample is essential, as otherwise there would no randomness in the estimated value function in initial periods. This procedure is the LSPI algorithm of \cite{lagoudakis2003least} applied with noisy data and a tabular representation. The paper \cite{osband2017deep} includes many experiments with non-tabular representations. 
\hrulefill
\begin{algorithm}[H]
	\SetNlSty{texttt}{(}{)}
	\SetAlgoLined
	\SetKwInOut{Input}{input}\SetKwInOut{Output}{output}
	\Input{$H$, $S$, $A$, tuning parameters $\{\beta_{k}\}_{k\in \mathbb{N}}$}
	\BlankLine
	\For{episodes $k=1,2,\ldots$ }{
		\tcc{Define squared temporal difference error}
		$\mathcal{L}(Q \mid Q_{\rm next}, \mathcal{D})= \sum_{(s,a,r,s')\in \mathcal{D}} \left( Q(s,a) - r- \max_{a'\in \Ac} Q_{\rm next}(s',a')\right)^2 $ \;   
		$\mathcal{D}_{h}=\{(s_{h}^{\ell}, a_{h}^{\ell}, r_{h}^{\ell}, s_{h+1}^{\ell} ) : \ell <k \} \qquad h<H$ 		\tcc*[r]{Past data}   
		$\mathcal{D}_{H}=\{(s_{H}^{\ell}, a_{H}^{\ell}, r_{H}^{\ell}, \emptyset ) : \ell <k \}$\;
		\BlankLine
		\tcc{Randomly perturb data}
		\For{time periods $h=1, \ldots, H$}{
			Sample array $\tilde{Q}_h \sim N(0, \beta_k I)$ \tcc*[r]{Draw prior sample}
			$\tilde{D}_{h} \leftarrow \{\}$\;
			\For{$(s,a,r,s')\in \mathcal{D}_h$}{
				sample $w\sim N(0,\beta_k)$\;
				$\tilde{\mathcal{D}}_h \leftarrow \tilde{\mathcal{D}}_h \cup \{(s,a,r+w,s')\}$\;
			}
		}
		\BlankLine	 
		\tcc{Estimate $Q$ on noisy data}
		Define terminal value $Q^{k}_{H+1}(s,a)\leftarrow 0 \quad \forall s,a$ \;
		\For{time periods $h = H,\ldots , 1$}{
			$\hat{Q}_{h} \leftarrow \argmin_{Q\in \mathbb{R}^{SA}}  \mathcal{L}(Q \mid Q_{h+1}, \tilde{\mathcal{D}}_{h}) + \|Q-\tilde{Q}_h \|_2^2$  \;
		}
		Apply greedy policy with respect to $ (\hat{Q}_{1}, \ldots \hat{Q}_H)$ throughout episode\;
		Observe data $s_{1}^k,a^k_1,r^{k}_1,\ldots s^k_{H}, a^k_H, r^k_H$ \;
	}
	\caption{RLSVI for Tabular, Finite Horizon, MDPs\label{alg: RLSVI}}
\end{algorithm}\DecMargin{1em}
To understand this presentation of RLSVI, it is helpful to understand an equivalence between posterior sampling in a Bayesian linear model and fitting a regularized least squares estimate to randomly perturbed data. We refer to \cite{osband2017deep} for a full discussion of this equivalence and review the scalar case here. 
Consider Bayes updating of a scalar parameter $\theta\sim N(0,\beta)$ based on noisy observations $Y=(y_{1},\ldots, y_{n})$ where $y_i\mid \theta \sim N(0, \beta)$. The posterior distribution has the closed form $\theta \mid Y \sim N\left( \frac{1}{n+1} \sum_{1}^{n} y_i  \,,\, \frac{\beta}{n+1} \right)$. We could generate a sample from this distribution by fitting a least squares estimate to noise. Sample $W=(w_1,\ldots, w_n)$ where each $w_i \sim N(0,\beta)$ is drawn independently and sample $\tilde{\theta} \sim N(0,\beta)$. Then 
\begin{equation}\label{eq: posterior sample}
\hat{\theta} = \argmin_{\theta \in \mathbb{R}}  \sum_{i=1}^{n} \left( \theta - y_i \right)^2 + (\theta - \tilde{\theta})^2 =  \frac{1}{n+1}\left( \sum_{i=1}^{n} y_i +\tilde{\theta} \right)
\end{equation} 
satisfies $\hat{\theta}\mid Y \sim N\left( \frac{1}{n+1} \sum_{1}^{n} y_i  \,,\, \frac{\beta}{n+1} \right)$.  For more complex models, where exact posterior sampling is impossible, we may still hope estimation on randomly perturbed data generates samples that reflect uncertainty in a sensible way.  As far as RLSVI is concerned, roughly the same calculation shows that in Algorithm \ref{alg: RLSVI} $\hat{Q}_{h}(s,a)$ is equal to an empirical Bellman update plus Gaussian noise: 
\[
\hat{Q}_{h}(s,a) \mid \hat{Q}_{h+1} \sim   N\left( \hat{R}_{h,s,a} + \sum_{s'\in \Sc} \hat{P}_{h,s,a}(s')  \max_{a'\in \Ac} \hat{Q}_{h+1}(s', a') \, , \, \frac{\beta_k}{n_{k}(h,s,a)+1} \right). 
\]


\section{Main result}
Theorem \ref{thm: main} establishes that RLSVI satisfies a worst-case polynomial regret bound for tabular finite-horizon MDPs. It is worth contrasting RLSVI to $\epsilon$--greedy exploration and Boltzmann exploration, which are both widely used randomization approaches to exploration. Those simple methods explore by directly injecting randomness to the action chosen at each timestep. Unfortunately, they can fail catastrophically even on simple examples with a finite state space -- requiring a time to learn that scales exponentially in the size of the state space. Instead, RLSVI generates randomization by training value functions with randomly perturbed rewards. Theorem \ref{thm: main} confirms that this approach generates a sophisticated form of exploration fundamentally different from $\epsilon$--greedy exploration and Boltzmann exploration. The notation $\tilde{O}$ ignores poly-logarithmic factors in $H,S,A$ and $K$. 
\begin{thm}\label{thm: main}
	Let $\Mc$ denote the set of MDPs with horizon $H$, $S$ states, $A$ actions, and rewards bounded in [0,1]. Then for a tuning parameter sequence $\beta=\{\beta_{k}\}_{k\in \mathbb{N}}$ with $\beta_k = \frac{1}{2}SH^3 \log(2HSAk)$, 
	\[
	\sup_{M \in \Mc} {\rm Regret}(M, K, \mathtt{RLSVI}_{\beta}) \leq \tilde{O}\left( H^{3}S^{3/2}\sqrt{AK} \right).
	\] 
\end{thm}
This bound  is not  state of the art and that is not the main goal of this paper. I conjecture that the extra factor of $S$ can be removed from this bound through a careful analysis, making the dependence on $S$, $A$, and $K$, optimal. This conjecture is supported by numerical experiments and (informally) by a Bayesian regret analysis \cite{osband2017deep}. One extra $\sqrt{S}$ appears to come from a step at the very end of the proof in Lemma \ref{lem: final regret decomposition}, where we bound a certain $L_1$ norm as in the analysis style of \cite{jaksch2010near}. For optimistic algorithms, some recent work has avoided directly bounding that $L_1$-norm, yielding a tighter regret guarantee    \cite{azar2017minimax, dann2017unifying}.  Another factor of $\sqrt{S}$ stems from the choice of $\beta_k$, which is used in the proof of Lemma \ref{lem: optimism temporary}. This seems similar to and extra $\sqrt{d}$ factor that appears in worst-case regret upper bounds for Thompson sampling in $d$-dimensional linear bandit problems \cite{abeille2017linear}.    
\begin{remark}
	Some translation is required to relate the dependence on $H$ with other literature. Many results are given in terms of the number of periods $T=KH$, which masks a factor of $H$. Also unlike e.g. \cite{azar2017minimax}, this paper treats time-inhomogenous transition kernels. In some sense agents must learn about $H$ extra state/action pairs. Roughly speaking then, our result exactly corresponds to what one would get by applying the UCRL2 analysis \cite{jaksch2010near} to a time-inhomogenous finite-horizon problem. 
\end{remark}

\section{Proof of Theorem \ref{thm: main}}
The proof follows from several lemmas. Some are (possibly complex) technical adaptations of ideas present in many regret analyses. Lemmas \ref{lem: optimism} and \ref{lem: reduction to prediction} are the main discoveries that prompted this paper. Throughout we use the following notation: for any MDP $\tilde{M}=(H, \Sc, \Ac, \tilde{P}, \tilde{R}, s_1)$, let $V(\tilde{M}, \pi)\in \mathbb{R}$ denote the value function corresponding to policy $\pi$ from the initial state $s_1$. In this notation, for the true MDP $M$ we have $V(M, \pi ) = V^{\pi}_{1}(s_1)$.  
\paragraph{A concentration inequality.}
Through a careful application of Hoeffding's inequality, one can give a high probability bound on the error in applying a Bellman update to the (non-random) optimal value function $V^*_{h+1}$. Through this, and a union bound, Lemma bounds \ref{lem: concentration} bounds the expected number of times the empirically estimated MDP falls outside the confidence set 
\begin{align*}
\Mc^k = \bigg\{ (H, \Sc, \Ac, P', R', s_1) :    \quad   \forall  (h,s,a)   | (R'_{h,s,a}-R_{h,s,a}) + \langle P'_{h,s,a} - P_{s,a,h} \, , \, V^*_{h+1}  \rangle |\,\,\,\, & \\
\leq \sqrt{e^{k}(h,s,a)} & \bigg\}
\end{align*}
where we define 
\[
\sqrt{e_{k}(h,s,a)} = 
  H\sqrt{ \frac{ \log\left( 2HSA k  \right) }{n_k(s,h,a)+1}}. 
\]
This set is a only a tool in the analysis and cannot be used by the agent  since $V^*_{h+1}$ is unknown. 
\begin{restatable}[Validity of confidence sets]{lemma}{concentration}\label{lem: concentration}
	$\sum_{k=1}^{\infty} \Prob\left( \hat{M}^k  \notin \mathcal{M}^k \right) \leq \frac{\pi^2}{6}.$ 
\end{restatable}

\paragraph{From value function error to on policy Bellman error.}
For some fixed policy $\pi$, the next simple lemma expresses the gap between the value functions under two MDPs in terms of the differences between their Bellman operators. Results like this are critical to many analyses in the RL literature. Notice the asymmetric role of $\tilde{M}$ and $\overline{M}$. The value functions correspond to one MDP while the state trajectory is sampled in the other. We'll apply the lemma twice: once where $\tilde{M}$ is the true MDP and $\overline{M}$ is estimated one used by RLSVI and once where the role is reversed. 
\begin{lemma}\label{lem: value gap to bellman} 
	Consider any policy $\pi$ and two MDPs $\tilde{M}=(H, \Sc, \Ac, \tilde{P}, \tilde{R}, s_1)$  and $\overline{M}=(H, \Sc, \Ac, \overline{P}, \overline{R}, s_1)$. 
 	Let $\tilde{V}^{\pi}_h$ and $\overline{V}^{\pi}_h$ denote the respective value functions of $\pi$ under $\tilde{M}$ and $\overline{M}$. Then 
	\[
	\overline{V}_1^{\pi}(s_1)-\tilde{V}_1^{\pi}(s_1) = \E_{\pi, \overline{M}}\left[ \sum_{h=1}^{H} \left(\overline{R}_{h, s_h, \pi(s_h)} - \tilde{R}_{h, s_h, \pi(s_h)}  \right) + \langle  \overline{P}_{h,s_h, \pi(s_h)}-\tilde{P}_{h,s_h, \pi(s_h)} \, , \, \tilde{V}^{\pi}_{h+1}  \rangle   \right],
	\]
	where $\tilde{V}^{\pi}_{H+1}\equiv 0\in \mathbb{R}^S$ and the expectation is over the sampled state trajectory $s_{1},\ldots s_H$ drawn from following $\pi$ in the MDP $\overline{M}$. 
\end{lemma}
\begin{proof}
	\begin{align*}
	&\overline{V}_1^{\pi}(s_1)-\tilde{V}_1^{\pi}(s_1)\\
	 =& \overline{R}_{1, s_1, \pi(s_1)} +\langle  \overline{P}_{1,s_1, \pi(s_1)} \, , \, \overline{V}^{\pi}_{2}  \rangle   -\tilde{R}_{1, s_1, \pi(s_1)}   -\langle  \tilde{P}_{1,s_1, \pi(s_1)} \, , \, \tilde{V}^{\pi}_{2}  \rangle \\
	=& \overline{R}_{1, s_1, \pi(s_1)} -R_{1, s_1, \pi(s_1)} +\langle  \overline{P}_{1,s_1, \pi(s_1)}-\tilde{P}_{1,s_1, \pi(s_1)} \, , \, \tilde{V}^{\pi}_{2}  \rangle +\langle  \overline{P}_{1,s_1, \pi(s_1)} \, , \,\overline{V}^{\pi}_{2} - \tilde{V}^{\pi}_{2}  \rangle  \\
	=& \overline{R}_{1, s_1, \pi(s_1)} -\tilde{R}_{1, s_1, \pi(s_1)} +\langle  \overline{P}_{1,s_1, \pi(s_1)}-\tilde{P}_{1,s_1, \pi(s_1)} \, , \, \tilde{V}^{\pi}_{2}  \rangle  + \E_{\pi, \overline{M}}\left[\overline{V}_2^{\pi}(s_2)-\tilde{V}_2^{\pi}(s_2) \right].
	\end{align*}
	Expanding this recursion gives the result. 
\end{proof}

\paragraph{Sufficient optimism through randomization.}\label{subsec: optimism}
There is always the risk that, based on noisy observations, an RL algorithm incorrectly forms a low estimate of the value function at some state. This may lead the algorithm to purposefully avoid that state, therefore failing to gather the data needed to correct its faulty estimate. To avoid such scenarios, nearly all provably efficient RL exploration algorithms are based build purposefully optimistic estimates. RLSVI does not do this, and instead generates a randomized value function. The following lemma is key to our analysis. It shows that, except in the rare event when it has grossly mis-estimated the underlying MDP, RLSVI has at least a constant chance of sampling an optimistic value function. Similar results can be proved for Thompson sampling with linear models \cite{abeille2017linear}. Recall $M$ is unknown true MDP with optimal $\pi^*$ and $\overline{M}^k$ is RLSVI's noise perturbed MDP under which $\pi^k$ is an optimal policy.

\begin{lemma}\label{lem: optimism} Let $\pi^*$ be an optimal policy for the true MDP $M$. If $\hat{M}^k \in \Mc^k$, then
	$
	\Prob\left( V(\overline{M}^k, \pi^{k}) \geq  V(M, \pi^*)\mid \hist \right) \geq \Phi(-1).
	$
\end{lemma}
This result is more easily established through the following lemma, which avoids the need to carefully condition on the history $\hist$ at each step. We conclude with the proof of Lemma \ref{lem: optimism} after.  
\begin{lemma}\label{lem: optimism temporary}
	Fix any policy $\pi=(\pi_1,\ldots, \pi_H)$ and vector $e\in \mathbb{R}^{HSA}$ with $e(h,s,a)\geq 0$. Consider the MDP $M=(H, \Sc, \Ac, P, R, s_1)$ and alternative $\bar{R}$ and $\bar{P}$ obeying the inequality  
	\[
	-\sqrt{e(h,s,a)} \leq \bar{R}_{h,s,a} - R_{h,s, a} + \langle \bar{P}_{h,s, a} - P_{h,s, a} , V_{h+1}\rangle  \leq \sqrt{e(h,s,a)}
	\]
	for every $s\in \Sc, a\in \Ac$ and $h\in \{1,\ldots, H\}$. Take $W\in \mathbb{R}^{HSA}$ to be a random vector with independent components where $w(h,s,a) \sim N(0, HSe(h,s,a))$. Let $\bar{V}^{\pi}_{1, W}$ denote the (random) value function of the policy $\pi$ under the MDP $\bar{M}=(H, \Sc, \Ac, \bar{P}, \bar{R}+W)$. Then 
	\[
	\Prob\left(\bar{V}^{\pi}_{1, W}(s_1) \geq  V^{\pi}_{1}(s_1)  \right) \geq \Phi(-1).
	\]
\end{lemma}
\begin{proof}
	To start, we consider an arbitrary deterministic vector $w\in \mathbb{R}^{HSA}$ (thought of as a possible realization of $W$) and evaluate the gap in value functions $\bar{V}^{\pi}_{1,w}(s_1) -V^{\pi}_1(s_1)$. We can re-write this quantity by applying Lemma \ref{lem: value gap to bellman}. Let $s=(s_{1},\ldots, s_H)$ denote a random sequence of states drawn by simulating the policy $\pi$ in the MDP $\bar{M}$ from the deterministic initial state $s_1$. Set $a_h=\pi(s_h)$ for $h=1,\ldots, H$. Then
	\begin{align*}
	\bar{V}^{\pi}_{1,w}(s_1) -V^{\pi}_1(s_1) &=  \E\left[  \sum_{h=1}^{H}   w(h,s_h,\pi_h(s_h)) +\bar{R}_{h,s_h, \pi_h(s_h)}-R_{h,s_h, \pi_h(s_h)} +\langle \bar{P}_{h,s_h, \pi_h(s_h)} - P_{h,s_h, \pi_h(s_h)} \, , \,  V^{\pi}_h \rangle\right]  \\
	&\geq H\E\left[ \frac{1}{H}\sum_{h=1}^{H}  \left( w(h,s_h,\pi_h(s_h)) - \sqrt{e(h,s_h, \pi_h(s_h))}  \right)  \right]\\\label{eq: optimism useful}
	\end{align*}
	where the expectation is taken over the sequence of sates $s=(s_1,\ldots, s_H)$. Define $d(h,s)=\frac{1}{H}\Prob(s_h=s)$ for every $h\leq H$ and $s\in \Sc$. Then the above equation can be written as
	\begin{align*}
		\frac{1}{H} \left(\bar{V}^{\pi}_{1,w}(s_1) -V^{\pi}_1(s_1)\right) &\geq  \sum_{s\in \Sc, h\leq H}  d(h,s) \left( w(h,s_h,\pi_h(s_h)) - \sqrt{e(h,s_h, \pi_h(s_h))} \right)\\
		&\geq    \left( \sum_{s\in \Sc, h\leq H}  d(h,s)  w(h,s_h,\pi_h(s_h)) \right) - \sqrt{HS} \sqrt{ \sum_{s\in \Sc, h\leq H} d(h,s)^2  e(h,s_h, \pi_h(s_h))} \\
		&:= X(w)
	\end{align*}
	where the second inequality applies Cauchy-Shwartz. Now, since 
	\[
	d(h,s)W(h,s,\pi_h(s,a))\sim N(0, d(h,s)^2 HSe(h,s,\pi_h(s,a))),
	\]
	we have 
	\[ 
	X(W) \sim N\left(  -  \sqrt{HS \sum_{s\in \Sc, h\leq H} d(h,s)^2  e(h,s_h, \pi_h(s_h))},  HS \sum_{s\in \Sc, h\leq H} d(h,s)^2  e(h,s_h, \pi_h(s_h)) \right). 
	\]
	By standardization,  $\Prob(X(W) \geq 0) = \Phi(-1)$. Therefore, $\Prob(\bar{V}^{\pi}_{1,w}(s_1) -V^{\pi}_1(s_1) \geq 0 ) \geq \Phi(-1)$. 
\end{proof}

\begin{proof}[Proof of Lemma \ref{lem: optimism}]
	Consider some history $\hist$ with $\hat{M}^k \in \Mc^k$. Recall $\pi^k$ is the policy chosen by RLSVI, which is optimal under the MDP $\overline{M}^k = (H, \Sc, \Ac, \hat{P}^k, \hat{R}^k + w^k, s_1)$. Since $\sigma_{k}(h,s,a)= HSe_{k}(h,s,a)$, applying Lemma \ref{lem: optimism temporary} conditioned on $\hist$ shows that with probability at least $\Phi(-1)$, $V(\overline{M}^k, \pi^*)\geq V(M, \pi^*)$. When this occurs, we always have $V(\overline{M}^k, \pi^k) \geq V(M, \pi^*)$, since by definition $\pi^k$ is optimal under $\overline{M}^{k}$. 
\end{proof}

\paragraph{Reduction to bounding online prediction error.}\label{subsec: reduction to online}
The next Lemma shows that the cumulative expected regret of RLSVI is bounded in terms of the total prediction error in estimating the value function of $\pi^k$. The critical feature of the result is it only depends on the algorithm being able to estimate the performance of the policies it actually employs and therefore gathers data about. From here, the regret analysis will follow only concentration arguments.  For the purposes of analysis, we let $\tilde{M}^k$ denote an imagined second sample drawn from the same distribution as the perturbed MDP $\overline{M}^k$ under RLSVI. More formally, let $\tilde{M}^{k}=(H,\Sc, \Ac, \hat{P}^k, \hat{R}^k + \tilde{w}^k, s_1)$ where $\tilde{w}^k(h,s,a)\mid \hist \sim N(0, \sigma_{k}^2(h,s,a))$ is independent Gaussian noise. Conditioned on the history, $\tilde{M}^{k}$ has the same marginal distribution as $\overline{M}^{k}$, but it is statistically independent of the policy $\pi^k$ selected by RLSVI, 
\begin{restatable}[]{lemma}{prediction}\label{lem: reduction to prediction}
For an absolute constant $c=\Phi(-1)^{-1}<6.31$, we have
\begin{align*}
	\regret(M, K, \mathtt{RLSVI}_{\beta}) \leq&  (c+1) \E\left[\sum_{k=1}^{K} |V(\overline{M}^k, \pi^k) - V(M, \pi^k)|  \right] \\
	&+ c \E\left[\sum_{k=1}^{K} |V(\tilde{M}^k, \pi^k) - V(M, \pi^k)|  \right] 
	+ H\underbrace{\sum_{k=1}^{K} \Prob(\hat{M}^k \notin \Mc^k )}_{\leq \pi^2/6}.
\end{align*}	
\end{restatable}

\paragraph{Online prediction error bounds.}
We complete the proof with concentration arguments. Set $\epsilon_R^{k}(h,s,a)= \hat{R}^{k}_{h,s,a}  - R_{h,s,a} \in \mathbb{R}$ and
$\epsilon_{P}^k(h,s,a) = \hat{P}^{k}_{h, s, a} - P_{h,s_h, a_h} \in \mathbb{R}^S$ to be the error in estimating mean the mean reward and transition vector corresponding to $(h,s,a)$. The next result follows by bounding each term in Lemma \ref{lem: reduction to prediction}. This is done by using Lemma \ref{lem: value gap to bellman} to expand the terms $V(\overline{M}, \pi^k)-V(M, \pi^k)$ and $V(\overline{M}, \pi^k)-V(\tilde{M}, \pi^k)$. We focus our analysis on bounding $\E\left[\sum_{k=1}^{K} |V(\overline{M}^k, \pi^k) - V(M, \pi^k)|  \right]$. The other term can be bounded in an identical manner\footnote{In particular, an analogue of Lemma 7 holds \ref{lem: final regret decomposition} holds where we replace $\overline{M}^k$ with $\tilde{M}^k$, $V^{k}_{h+1}$ with the value function $\tilde{V}^{k}_{h+1}$ corresponding to policy $\pi^k$ in the MDP $\tilde{M}^k$, and the Gaussian noise $w^k$ with the fictitious noise terms $\tilde{w}^k$.}, so we omit this analysis. 
\begin{restatable}[]{lemma}{finalRegretDecomposition}\label{lem: final regret decomposition}
Let $c=\Phi(-1)^{-1}< 6.31$. Then for any $K \in \mathbb{N}$,
	\begin{eqnarray*}\E\left[\sum_{k=1}^{K} |V(\overline{M}^k, \pi^k) - V(M, \pi^k)|  \right]
&\leq& \sqrt{\E \sum_{k=1}^{K} \sum_{h=1}^{H-1}  \left\|   \epsilon_P^{k}(h, s^k_h, a^k_h) \right\|_1^2  }\,  \sqrt{\E \sum_{k=1}^{K} \sum_{h=1}^{H-1} \left\| V^{k}_{h+1}\right\|_{\infty}^2 } \\
		&&+\E\left[\sum_{k=1}^{K} \sum_{h=1}^{H} |\epsilon_R^{k}(h, s^k_h, a^k_h)| \right] + \E\left[\sum_{k=1}^{K} \sum_{h=1}^{H} |w^{k}(h, s^k_h, a^k_h)| \right].
	\end{eqnarray*}
\end{restatable}

The remaining lemmas complete the proof. At each stage, RLSVI adds Gaussian noise with standard deviation no larger than $\tilde{O}(H^{3/2}\sqrt{S})$. Ignoring extremely low probability events, we expect, $\left\| V^{k}_{h+1}\right\|_{\infty}\leq \tilde{O}(H^{5/2}\sqrt{S})$ and hence  $\sum_{h=1}^{H-1} \left\| V^{k}_{h+1}\right\|_{\infty}^2  \leq \tilde{O}(H^{6}S)$. The proof of this Lemma makes this precise by applying appropriate maximal inequalities. 
\begin{restatable}[]{lemma}{vfuncNorm}\label{lem: value function norm bound}
	\[
	\sqrt{\E \sum_{k=1}^{K} \sum_{h=1}^{H-1} \left\| V^{k}_{h+1}\right\|_{\infty}^2 }  = \tilde{O}\left( H^{3}\sqrt{SK} \right) 
	\]
\end{restatable}
The next few lemmas are essentially a consequence of analysis in \cite{jaksch2010near}, and many subsequent papers. We give proof sketches in the appendix. The main idea is to apply known concentration inequalities to bound  $\left\|  \epsilon_{P}^k(h,s,a)  \right\|_1^2  $, $ |\epsilon_R^{k}(h, s^k_h, a^k_h)|$ or $|w^{k}(h, s^k_h, a^k_h)|$ in terms of either $1/n_k(h,s_h, a_h)$ or $1/\sqrt{n_k(h,s_h, a_h)}$. The pigeonhole principle gives  $\sum_{k=1}^{K} \sum_{h=1}^{H-1}  1/n_k(h,s_h, a_h) = O(\log(SAKH)$ and $\sum_{k=1}^{K} \sum_{h=1}^{H-1}  (1/\sqrt{n_k(h,s_h, a_h)}) = O(\sqrt{SAKH})$ . 

\begin{restatable}[]{lemma}{perrors}\label{lem: perrors}
	\[
	\E\left[ \sum_{k=1}^{K} \sum_{h=1}^{H-1}  \left\|  \epsilon_{P}^k(h,s,a)  \right\|_1^2    \right]
	=\tilde{O}\left( S^2 A H \right)
	\]
\end{restatable}

\begin{restatable}[]{lemma}{rerrors}\label{lem: rerrors}
	\[
	\E\left[\sum_{k=1}^{K} \sum_{h=1}^{H} |\epsilon_R^{k}(h, s^k_h, a^k_h)|\right] = \tilde{O} \left(\sqrt{SA KH}\right)
	\]
\end{restatable}

\begin{restatable}[]{lemma}{werrors}\label{lem: werrors}
	\[
	\E\left[\sum_{k=1}^{K} \sum_{h=1}^{H} |w^{k}(h, s^k_h, a^k_h)|\right] = \tilde{O} \left(H^{3/2}S\sqrt{A KH}\right)
	\]
\end{restatable}

%

%

\paragraph{Acknowledgments.} Much of my understanding of randomized value functions comes from a collaboration with Ian Osband, Ben Van Roy, and Zheng Wen. Mark Sellke and Chao Qin each noticed the same error in the proof of Lemma \ref{lem: reduction to prediction} in the initial draft of this paper. The lemma has now been revised. I am extremely grateful for their careful reading of the paper. 

\bibliographystyle{plainnat}
\bibliography{references}

\newpage

\appendix

\section{Omitted Proofs}
\subsection{Proof of Lemma \ref{lem: concentration}}
\concentration* 
\begin{proof}
	The following construction is the standard way concentration inequalities are applied in bandit models and tabular reinforcement learning. See the discussion of what \citet{lattimore2018bandit} calls a ``stack of rewards'' model in Subsection 4.6. 	
	
	For every tuple $z=(h,s,a)$, generate two i.i.d sequences of random variables $r_{z,n}\sim \Rc_{h,s,a}$ and $s_{z,n}\sim P_{h,s,a}(\cdot)$. Here  $r_{(h,s,a),n}$ denotes the reward and $s_{(h,s,a),n}$ denotes the state transition generated from the $n$th time action $a$ is played in state $s$, period $n$.  Set
	\[ 
	Y_{z,n} = r_{z,n} + V_{h+1}^*(s_{z,n}) \qquad n\in \mathbb{N}.
	\]
	These are i.i.d, with $Y_{z,n} \in [0,H]$ since $\| V_{h+1}^*\|_{\infty } \leq H-1$, and satisfies 
	\[ 
	\E[Y_{z,n}] = R_{h,s,a} + \langle P_{h,s,a} \, , \, V^*_{h+1} \rangle.
	\]  
	By Hoeffding's inequality, for any $\delta_n \in (0,1)$,
	\[ 
	\Prob\left(  \left| \frac{1}{n}\sum_{i=1}^{n} Y_{(h,s,a),i}  - R_{h,s,a} - \langle P_{h,s,a} \, , \, V^*_{h+1} \rangle \right|  \geq  H\sqrt{\frac{\log(2/\delta_n)}{2n}} \right)  \leq \delta_n.
	\]
	For $\delta_n=\frac{1}{HSAn^2}$, a union bound over $HSA$ values of $z=(h,s,a)$ and all possible $n$ gives 
	\[ 
	\Prob\left( \bigcup_{h,s,a, n} \left\{ \left| \frac{1}{n}\sum_{i=1}^{n} Y_{(h,s,a),i}  - R_{h,s,a} - \langle P_{h,s,a} \, , \, V^*_{h+1} \rangle \right|  \geq  H\sqrt{\frac{\log(2/\delta_n)}{2n}}\right\} \right)  \leq  \sum_{n=1}^{\infty} \frac{1}{n^2}  = \frac{\pi^2}{6}.
	\]
	Now, by definition, if $n_{k}(h,s,a)=n>0$, we have 
	\[
	\hat{R}^{k}_{h,s,a} + \langle \hat{P}^{k}_{h,s,a}  \, ,\, V^*_{h+1} \rangle = \frac{1}{n} \sum_{i=1}^{n} Y_{(h,s,a), i}.
	\]
	Therefore, the above shows 
	\[
	\Prob\left( \exists (k,h,s,a)   \, : \, n_{k}(h,s,a)>0 \, , \, \left| \hat{R}^{k}_{h,s,a} -R_{h,s,a}+ \langle \hat{P}^{k}_{h,s,a} - P_{h,s,a} \, ,\, V^*_{h+1} \rangle    \right| \geq  H\sqrt{ \frac{ \log\left( 2HSA n_{k}(h,s,a)  \right) }{2n_k(s,h,a)}}  \right) 
	\]
	is upper bounded by $\pi^2/6$. Note that by definition, when $n_{k}(h,s,a)>0$ we have \[
	\sqrt{e^{k}(h,s,a)} \geq  H\sqrt{ \frac{ \log\left( 2HSA n_{k}(h,s,a)  \right) }{2n_k(s,h,a)}}
	\]
	and hence this concentration inequality holds with $\sqrt{e^k(h,s,a)}$ on the right hand side. 	
	When $n_{k}(h,s,a)=0$, we have the trivial bound 
	\[ 
	\left| \hat{R}^{k}_{h,s,a} -R_{h,s,a}+ \langle \hat{P}^{k}_{h,s,a} - P_{h,s,a} \, ,\, V^*_{h+1} \rangle    \right|  = | R_{h,s,a} +\langle P_{h,s,a} \, ,\, V^*_{h+1} \rangle   | \leq H \leq e^{k}(h,s,a)
	\]
	since we have defined the empirical estimates to satisfy $\hat{R}^{k}_{h,s,a}=0$ and $\hat{P}^{k}_{h,s,a}(\cdot)=0$ in the case that $h,s,a$ has never been played.
\end{proof}

\subsection{Proof of Lemma \ref{lem: reduction to prediction}}
\prediction*
\begin{proof}Recall that $\hist=\{(s^{i}_h, a^{i}_h, r^{i}_h) : h=1,\ldots H, i=1,\ldots, k-1\}$. So conditioned on $\hist$, $\overline{M}^k, \pi^k$ and $\tilde{M}^k$ are random only due to the internal randomness of the RLSVI algorithm. 
	 Set $\E_{k}[\cdot] = \E[\cdot \mid \hist]$. Suppose that $\hat{M}^k \in \Mc^k$. Then
	\begin{equation}\label{eq: rtp 1}
	 \Prob\left(V(\overline{M}^k, \pi^k) \geq V(M, \pi^*)  \bigg\vert \Hc_{k-1}  \right) \geq \Phi(-1).
	\end{equation}
	We begin with the regret decomposition:
	\begin{equation}\label{eq: rtp 2}
	\E_{k}\left[ V(M, \pi^*) - V(M, \pi^k)  \right] = \E_{k}\left[ V(M, \pi^*) - V(\overline{M}^k, \pi^k)  \right]+\E_{k}\left[ V(\overline{M}^k, \pi^k) - V(M, \pi^k) \right].  
	\end{equation}
	We focus on the first term. We show 
		\begin{equation}\label{eq: rtp 3}
 V(M, \pi^*) - \E_{k}\left[V(\overline{M}^k, \pi^k)  \right]\leq c \E_{k}\left[ \left( V(\overline{M}^k, \pi^k) -   \E_{k}\left[V(\overline{M}^k, \pi^k)  \right] \right)^+   \right].  
	\end{equation}
	The inequality is immediate if $V(M, \pi^*) < \E_{k}\left[V(\overline{M}^k, \pi^k)\right]$. We now show this when 
	$a\equiv V(M, \pi^*) - \E_{k}\left[V(\overline{M}^k, \pi^k)\right] \geq 0$. Then,
	\begin{eqnarray*}
	\E_{k}\left[ \left( V(\overline{M}^k, \pi^k) -   \E_{k}\left[V(\overline{M}^k, \pi^k)  \right] \right)^+   \right] &\geq& a \Prob_{k}\left( V(\overline{M}^k, \pi^k) -   \E_{k}\left[V(\overline{M}^k, \pi^k)  \right] \geq a \right) \\
	&=& \left( V(M, \pi^*) - \E_{k}\left[V(\overline{M}^k, \pi^k)\right]  \right) \Prob_{k}\left( V(\overline{M}^k, \pi^k) \geq V(M, \pi^*)  \right)\\
	&\geq & \left( V(M, \pi^*) - \E_{k}\left[V(\overline{M}^k, \pi^k)\right]  \right) \Phi(-1),
	\end{eqnarray*}	
	where the first step applies Markov's inequality, the second simply plugs in for $a$, and the third uses Equation \ref{eq: rtp 1}. Dividing each side by $\Phi(-1)$ gives Equation \eqref{eq: rtp 3}. Hence we have shown 
	\begin{equation}\label{eq: rtp 4}
	\E_{k}\left[ V(M, \pi^*) - V(M, \pi^k)  \right] \leq c \E_{k}\left[ \left( V(\overline{M}^k, \pi^k) -   \E_{k}\left[V(\overline{M}^k, \pi^k)  \right] \right)^+\right]  + \E_{k}\left[ V(\overline{M}^k, \pi^k) - V(M, \pi^k) \right]. 
	\end{equation}
	We complete our argument by bounding $\E_{k}\left[ \left( V(\overline{M}^k, \pi^k) -   \E_{k}\left[V(\overline{M}^k, \pi^k)  \right] \right)^+\right]$. For each fixed (nonrandom) policy $\pi$, define 
	\[ 
	\mu(\pi) \equiv \E_{k}\left[ V( \tilde{M}^k, \pi)\right]=\E_{k}\left[ V( \overline{M}^k, \pi)\right].
	\]
	Notice that $\mu(\pi^k) = \E_{k}\left[ V( \tilde{M}^k, \pi^k) \mid \pi^k \right]$ almost surely. This relies on the fact that $\tilde{M}^k$ and $\pi^k$ are independent conditioned on the history $\hist$. In general $\mu(\pi^k)  \neq  \E_{k}\left[ V( \overline{M}^k, \pi^k) \mid \pi^k\right]$, since $\pi^k$ is the optimal policy under $\overline{M}^k$ and so these two are statistically dependent. Now, for every policy $\pi$
	\begin{equation*}
	\mu(\pi) = \E_{k}\left[ V( \overline{M}^K, \pi)\right] \leq \E_{k}\left[\sup_{\pi'} V( \overline{M}^K, \pi')\right] = \E_{k}\left[ V( \overline{M}^K, \pi^k)\right].
	\end{equation*}
	So, $
	\mu(\pi^k) \leq \E_{k}\left[ V( \overline{M}^K, \pi^k)\right]$ almost surely. Using this, we find 
	\begin{eqnarray*}
	\E_{k}\left[ \left( V(\overline{M}^k, \pi^k) -   \E_{k}\left[V(\overline{M}^k, \pi^k)  \right] \right)^+\right] &\leq& \E_{k}\left[ \left( V(\overline{M}^k, \pi^k) -    \mu(\pi^k)\right)^+\right] \\
	&\leq& \E_{k}\left[ \left| V(\overline{M}^k, \pi^k) -    \mu(\pi^k)\right|\right] \\
	&=& \E_{k}\left[  \left| V(\overline{M}^k, \pi^k) -    \E_{k}\left[ V(\tilde{M}^k, \pi^k) \mid \pi^k, \overline{M}^k \right]    \right|  \,     \right]\\
	&\leq & \E_{k}\left[ \E_{k}\left[  \left| V(\overline{M}^k, \pi^k) -     V(\tilde{M}^k, \pi^k)     \right|   \, \bigg\vert \, \pi^k ,\overline{M}^k \right]      \right]\\
	&=& \E_{k}\left[  \left| V(\overline{M}^k, \pi^k) -     V(\tilde{M}^k, \pi^k)     \right| \right] \\
	&\leq& \E_{k}\left[  \left| V(\overline{M}^k, \pi^k) -     V(M, \pi^k)     \right| \right]+\E_{k}\left[  \left| V(\tilde{M}^k, \pi^k) -     V(M, \pi^k)     \right| \right]. 
	\end{eqnarray*}
	Plugging this into \eqref{eq: rtp 4} shows that, for any history $\hist$ with $\hat{M}^k \in \Mc^k$, 
	\[
	\E_{k}\left[ V(M, \pi^*) - V(M, \pi^k)  \right] \leq (c+1)\E_{k}\left[  \left| V(\overline{M}^k, \pi^k) -     V(M, \pi^k)     \right| \right]+c\E_{k}\left[  \left| V(\tilde{M}^k, \pi^k) -     V(M, \pi^k)     \right| \right].
	\]
	In the unlikely event $\hat{M}^k \in \Mc^k$, we have the worst case bound
	\[
	0\leq V(M, \pi^* ) - V(M, \pi^k) \leq H. 
	\]
	Combing these two cases and taking expectations gives   
	\[
	\E\left[ V(M, \pi^*) - V(M, \pi^k)  \right] \leq H\Prob(\hat{M}^k \notin \Mc^k) +(c+1)\E\left[  \left| V(\overline{M}^k, \pi^k) -     V(M, \pi^k)     \right| \right]+c\E\left[  \left| V(\tilde{M}^k, \pi^k) -     V(M, \pi^k)     \right| \right].
	\]
	Summing over $k$ concludes the proof. 
\end{proof}

\subsection{Proof of Lemma \ref{lem: final regret decomposition}}
\finalRegretDecomposition*
\begin{proof}
	We bound each term in the bound in Lemma \ref{lem: reduction to prediction}. By applying Lemma \ref{lem: value gap to bellman} with a choice of $\overline{M}=M$ and $\tilde{M}=\overline{M}^K$, the largest term is bounded, for any $k\in \mathbb{N}$, as
	\begin{eqnarray*}
		&&\left| V(\overline{M}^k, \pi^k) - V(M, \pi^k)\right| \\ 
		&=& \left| \E\left[ \sum_{h=1}^{H}\left(  \langle \hat{P}^{k}_{h, s^k_h, a^k_h} - P_{h,s^k_h, a^k_h} \, , \, V^{k}_{h+1} \rangle  \right) +\hat{R}^{k}_{h,s_h^k,a_h^k}+ w^{k}(h, s^k_h, a^k_h) - R_{h,s_h^k,a_h^k} \, \bigg\vert \,  \pi^k, \hist \right] \right|   \\
		&\leq&  \E\left[ \sum_{h=1}^{H-1} \left\|    \epsilon_P^{k}(h, s^k_h, a^k_h) \right\|_1 \left\| V^{k}_{h+1}\right\|_{\infty} \, \bigg\vert \,  \pi^k, \hist \right]  + \E\left[  \sum_{h=1}^{H} \left(|\epsilon_R^{k}(h, s^k_h, a^k_h)| +|w^{k}(h, s^k_h, a^k_h)| \right) \, \bigg\vert \,  \pi^k, \hist  \right]
	\end{eqnarray*}
	Taking expectations, summing over $k$, and applying Cauchy-Schwartz gives
	\begin{eqnarray*}
		\E\left[ \sum_{k=1}^{K} \left| V(\overline{M}^k, \pi^k) - V(M, \pi^k)\right|  \right] &\leq& \sqrt{\E \sum_{k=1}^{K} \sum_{h=1}^{H-1}  \left\|   \epsilon_P^{k}(h, s^k_h, a^k_h)\right\|_1^2  } \sqrt{\E \sum_{k=1}^{K} \sum_{h=1}^{H-1} \left\| V^{k}_{h+1}\right\|_{\infty}^2 } \\
		&&+ \E\left[\sum_{k=1}^{K} \sum_{h=1}^{H} |\epsilon_R^{k}(h, s^k_h, a^k_h)| \right]+ \E\left[\sum_{k=1}^{K} \sum_{h=1}^{H} |w^{k}(h, s^k_h, a^k_h)|\right].
	\end{eqnarray*}
\end{proof}

\subsection{Proof of Lemma \ref{lem: value function norm bound}}
The proof relies on the following maximal inequality.
\begin{lemma}[Example 2.7 From \cite{boucheron2013concentration}]\label{lem: maximal inequality}
	If $X_1,\ldots, X_n$ are i.i.d. random variables following a $\chi^2_1$ distribution, then 
	\[ 
	\E\left[ \max_{i\leq n} X_i \right] \leq 1+ \sqrt{2\log(n)} + 2\log(n).	
	\]
\end{lemma}
Let us now recall Lemma \ref{lem: value function norm bound}.
\vfuncNorm*
	\begin{proof}
		We have
		\[ 
		\sqrt{\E \sum_{k=1}^{K} \sum_{h=1}^{H-1} \left\| V^{k}_{h+1}\right\|_{\infty}^2 }  \leq \sqrt{HK  \E\left[ \max_{k\leq K,h\leq H} \| V^k_{h+1}\|_{\infty}^2 \right]}
		\]
		Now
		\begin{eqnarray*}
			V^k_{h+1}(s')\leq (H-h-1)(1+\max_{h,s,a}  w^{k}(h,s,a))\\
			V^k_{h+1}(s')\geq (H-h-1)(\min_{h,s,a}  w^{k}(h,s,a)).
		\end{eqnarray*}
		Together this gives that for all $k\leq K$ and $h\in \{1,\ldots, H-1\}$
		\[ 
		\| V^k_{h+1}\|_{\infty} \leq  H\left(1+\max_{k\leq K, h,s,a} | w^{k}(h,s,a)| \right)^2 \leq 4H^2+4H^2\left( \max_{k\leq K, h,s,a} | w^{k}(h,s,a)|^2 \right).
		\]
		We have $w^k(h,s,a) = \sigma_{k}(h,s,a) \xi^k_{h,s,a}$ where the $\xi^k_{h,s,a}\sim N(0,1)$ are drawn i.i.d across $h,s,a$. Set $X^k_{h,s,a}=(\xi^k_{h,s,a})^2$, each of which follows a chi-squared distribution with 1 degree of freedom. Then, 
		\begin{eqnarray*}
			\E\left[\max_{k\leq K, h,s,a}  |w^{k}(h,s,a)|^2 \right] &\leq& \left(\max_{k\leq K, h,s,a,}\sigma^2_{k}(h,s,a)\right) \E\left[ \max_{k\leq K, h,s,a} |\xi^k_{h,s,a}|^2 \right]\\ 
			&=& \left(\max_{k\leq K, h,s,a,}\sigma^2_{k}(h,s,a)\right) \E\left[ \max_{k\leq K, h,s,a} X^k_{h,s,a} \right] \\
			&\leq&  \left( SH^3 \log(2SAHK)  \right)  \E\left[ \max_{k\leq K, h,s,a} X^k_{h,s,a} \right] \\
			&\leq & \left( SH^3 \log(2SAHK)  \right)\left( 1+\sqrt{2\log(SAHK)} +2\log(SAHK)\right)  \\
			&\leq & O\left(SH^{3}\log\left(2SAHK\right)^2 \right). 
		\end{eqnarray*}
		This gives us 
		\[
		\sqrt{KH\E\left[ \max_{k\leq K,h\leq H} \| V^k_{h+1}\|_{\infty}^2 \right]} = \tilde{O}\left(\sqrt{KH\cdot H^2 \cdot S H^{3}}\right) =\tilde{O}\left( H^3 \sqrt{SK} \right). 
		\]
	\end{proof}

\subsection{Proof sketch of Lemma \ref{lem: perrors}}
This result relies on an inequality by \citet{weissman2003inequalities}, which we now restate.  
\begin{lemma}\label{lem: weissman} [L1 deviation bound]
	If $p$ is a probability distribution over $\Sc=\{1,\ldots S\}$ and $\hat{p}$ is the empirical distribution constructed from $n$ i.i.d draws from $p$, then for any $\epsilon>0$, 
	\[ 
	\Prob\left( \left\|  \hat{p} - p \right\|_1 \geq \epsilon \right)  \leq (2^{S} -2) \exp\left( -\frac{n\epsilon^2}{2} \right)  
	\]
\end{lemma}
\perrors*
\begin{proof}[Proof sketch]
	By picking an appropriate $\epsilon$ in Lemma \ref{lem: weissman} as in \cite[Appendix C.1]{jaksch2010near}, together with a union bound over all $HSA$ possible values for the tuple $(h,s,a)$, there exists a numerical constant $c$ such that 
	\begin{equation}\label{eq: p union bound}
	\Prob\left( \bigcup_{s, a, h, k\leq K} \left\{  \left\|  \hat{P}^k_{h,s,a}-P_{h,s,a}   \right\|_{1}  \geq c\sqrt{\frac{S\log(1+HSAK)}{n_{k}(h,s,a)}+1} \right\} \right) \leq \frac{1}{K H}. 
	\end{equation}
	Set $\beta_{k}(h,s,a)= \frac{S\ell}{n_{k}(h,s,a)}$ where $\ell=c^2\log(1+HSAK)$ denotes a logarithmic factor. Recall the definition $\epsilon^p_{k}(h,s,a) \equiv  \hat{P}^k_{h,s,a}-P_{h,s,a}$. Let $B$ be the ``bad event'' that $\|\epsilon^p_{k}(h,s,a)\|_1^2 \geq  \beta_{k}(h,s,a)$ for some $(h,s,a)$ and $k\leq K$. Since $\| \epsilon^p_{k}(h,s,a) \|_1\leq 2$ always, we have 
	\begin{equation}\label{eq: p contribution of bad event}
	\E \sum_{k=1}^{K}\sum_{h=1}^{H-1} \| \epsilon^k_P(h, s^k_h, a^k_h)  \|_1^2 \mathbf{1}(B) \leq 4
	\end{equation}
	On the other hand, assuming $B^{c}$ we have the bound 
	\begin{align*}
	\sum_{k=1}^{K}\sum_{h=1}^{H-1}  \| \epsilon^k_P(h, s^k_h, a^k_h)  \|_1^2 \leq  \sum_{k=1}^{K} \sum_{h=1}^{H-1} \beta_{k}(h,s,a) &= S\ell  \sum_{k=1}^{K} \sum_{h=1}^{H-1} \frac{1}{n_{k}(h, s_h, a_h)+1} \\
	&\leq \sum_{h,s,a} \sum_{n=0}^{n_{K}(h,s,a)} \frac{1}{n+1}\\
	&= O\left(HSA\log(K)\right).  
	\end{align*}	
\end{proof}

\subsection{Proof sketch of Lemma \ref{lem: rerrors}}
\rerrors*
\begin{proof}[Proof sketch]
		The proof is similar to Lemma \ref{lem: perrors}. By Hoeffding's inequality together with a union bound, we can ensure that $|\epsilon_R^k(h,s,a)|\leq c\sqrt{ \frac{\log(1+HSAK)}{n_{k}(h,s,a)+1}}$ for all $k\leq K$ and all tuples $(h,s,a)$ except on some bad event that, as in \eqref{eq: p contribution of bad event}, contributes at most a constant to the bound. Now the result follows from using the pigeonhole principle to conclude
		\[
		\sum_{k=1}^{K} \sum_{h=1}^{H} \frac{1}{\sqrt{n_{k}(h,s,a)}} = O \left( \sqrt{HSAK} \right). 
		\]
		This kind of bound bound is standard in the RL and bandit literature. See \cite[Appendix A]{osband2013more} for one proof. 
\end{proof}

\subsection{Proof sketch of Lemma \ref{lem: werrors}}
\werrors*
\begin{proof}
Recall $\sigma_{k}(h,s,a) = \sqrt{\frac{\beta}{n_{k}(h,s,a)+1}}$ where $\beta_k=\tilde{O}(SH^3)$. Write $w_k(h,s,a) = \sigma_{k}(h,s,a) \xi_k({h,s,a})$ where $\xi_k({h,s,a}\sim N(0,1)$ and the array of random variable $\{  \xi_k({h,s,a}) : 1\leq k\leq K, 1\leq h\leq H, a\in \Ac, s\in \Sc \}$ is drawn independently. By Holder's inequality,  
\[  
\E\sum_{k=1}^{K} \sum_{h=1}^{H} |w_k(h,s_h,a_h)| \leq \E \left(\max_{k\leq K, h,s,a} |\xi_k({h,s,a}| \right)  \E \sum_{k=1}^{K} \sum_{h=1}^{H} \sigma_{k}(h,s_h, a_h)
\]
The (sub) Gaussian maximal inequality gives 
\[
\E \left(\max_{k\leq K, h,s,a} |\xi_k({h,s,a}| \right) = O\left(\sqrt{\log(HSAK)} \right).
\] 
To simplify the next expression, note that $\beta_k \leq \beta_K$. On any sample path, by the same argument as in Lemma \ref{lem: rerrors}, we have 
\[ 
\sum_{k=1}^{K} \sum_{h=1}^{H} \sigma_{k}(h,s_h, a_h) \leq \beta_K\sum_{k=1}^{K} \sum_{h=1}^{H} \sqrt{ \frac{1}{n_{k}(h,s_h,a_h)+1}} = O\left( \beta_K \sqrt{HSAK} \right).
\]
\end{proof}
\end{document}